\documentclass[10pt, conference, letterpaper]{ IEEEtran}
\IEEEoverridecommandlockouts

\pdfoutput=1 
\usepackage{hyperref}
\usepackage[all]{hypcap} 
\usepackage{color,overpic,cite}
\usepackage{xcolor}
\usepackage{pdfsync}
\usepackage{url}
\usepackage{amsmath,amsfonts,amssymb,amsthm,bbm,bm,comment,multirow, subfigure}
\usepackage{tcolorbox}
\usepackage{siunitx}
\usepackage{mdframed}
\usepackage{bm}
\usepackage[ruled,commentsnumbered]{algorithm2e}
\usepackage{tabularx,booktabs}
\usepackage[normalem]{ulem}
\usepackage{setspace}
\usepackage{enumitem}
\usepackage{amssymb} 
\usepackage{pifont}
\usepackage{algorithmic}
\usepackage{dsfont}
\usepackage{authblk} 
\usepackage[cal=cm,scr=euler]{mathalpha}
\usepackage{float} 
\hypersetup{nolinks=true} 

\newtheorem{lemma}{Lemma}

\renewcommand{\baselinestretch}{0.96} 

\def\BibTeX{{\rm B\kern-.05em{\sc i\kern-.025em b}\kern-.08em
    T\kern-.1667em\lower.7ex\hbox{E}\kern-.125emX}}

\begin{document}

\title{CHOMET: Conditional Handovers via Meta-Learning}

\author{Michail Kalntis, Fernando A. Kuipers, and George Iosifidis \vspace{1mm}}

\affil{Delft University of Technology, The Netherlands \\ Emails: \{m.kalntis, f.a.kuipers, g.iosifidis\}@tudelft.nl}


\maketitle

\begin{abstract}
Handovers (HOs) are the cornerstone of modern cellular networks for enabling seamless connectivity to a vast and diverse number of mobile users. However, as mobile networks become more complex with more diverse users and smaller cells, traditional HOs face significant challenges, such as prolonged delays and increased failures. To mitigate these issues, 3GPP introduced conditional handovers (CHOs), a new type of HO that enables the preparation (i.e., resource allocation) of multiple cells for a single user to increase the chance of HO success and decrease the delays in the procedure.
Despite its advantages, CHO introduces new challenges that must be addressed, including efficient resource allocation and managing signaling/communication overhead from frequent cell preparations and releases. This paper presents a novel framework aligned with the O-RAN paradigm that leverages meta-learning for CHO optimization, providing robust dynamic regret guarantees and 
demonstrating at least 180\% superior performance than other 3GPP benchmarks in volatile signal conditions.

\end{abstract}


\IEEEpeerreviewmaketitle

\section{Introduction}\label{sec:intro}

\subsection{Motivation}

\IEEEPARstart{T}{he} heterogeneity of modern mobile networks, characterized by multiple radio access technologies, dense cell deployments, and high frequencies, introduces significant challenges in mobility management that might lead to increased delays, service interruptions, and Quality of Service (QoS) degradation \cite{kalntis_imc24}. The handover (HO) mechanism is key to addressing these issues by enabling smooth transitions between cells as users move and stream, browse, text, or work \cite{3gpp_38_300}.
Although \textit{traditional} HO mechanisms have existed for decades, it has been shown that, even in 
new radio access technologies (5G and beyond), HOs still face many challenges. The reason lies in the HO procedure itself, as it relies on \textit{reactive} decision-making: a HO is initiated only when the signal quality of the serving cell deteriorates significantly \cite{3gpp_38_331}. This approach often leads to increased HO delays and failures (HOFs), particularly in dense and high-frequency environments where signal conditions fluctuate rapidly \cite{kalntis_imc24, stanczak22}.

To mitigate these issues, 3GPP recently introduced Conditional Handovers (CHOs) \cite{3gpp_38_300}. The main idea lies in reserving resources (e.g., physical resource blocks and spectrum) to multiple \textit{candidate} cells for a single user \textit{proactively}, namely, when signal conditions are still favorable (in contrast to traditional HOs that happen when conditions are poor), and offloading the final HO decision to the user.
The novel idea of \textit{network-configured, user-decided HOs}, in contrast to the traditional \textit{network-configured, user-assisted} HOs, has recently been shown to significantly reduce HO delays and HOFs, rendering them a promising solution in different environments, such as 5G beamformed
\cite{iqbal_beamforming23} or non-terrestrial networks \cite{leo_juan22}. Nevertheless, CHOs introduce new challenges that must be carefully studied to ensure seamless and optimal mobility management in heterogeneous cellular networks.
Specifically, the key difficulty lies in determining which cells to prepare while signal conditions vary, as preparing too many leads to resource waste and increased \textit{signaling} due to extra communications, while too few might lead to traditional HO \cite{prado_CHOopt23, stanczak22, iqbal22, martikainen18}.

Figure \ref{fig:CHO_scenarios} illustrates these challenges: if at the intersection of the three cells (i.e., at $t_3$), we prepare only $C_2$ ($C_3$) but the user moves towards $C_3$ ($C_2$), then $C_2$ ($C_3$) resources will be wasted; at $t_4$, extra signaling/communication will be needed to release these resources and a traditional HO will occur. Preparing both $C_2$ and $C_3$, if they have enough capacity, is beneficial to the user, and soon after (i.e., at $t_4$) the trajectory becomes clear and a CHO is executed, the resources allocated to the unneeded cell can be released, creating, however, some signaling cost. Ideally, if only $C_2$ ($C_3$) is prepared at $t_3$ and the user's movement towards $C_2$ ($C_3$) is confirmed at $t_4$, the CHO can be executed efficiently, ensuring optimal resource allocation with minimal signaling overhead.




\begin{figure}[t]
    \centering
    \includegraphics[width=0.65\columnwidth]{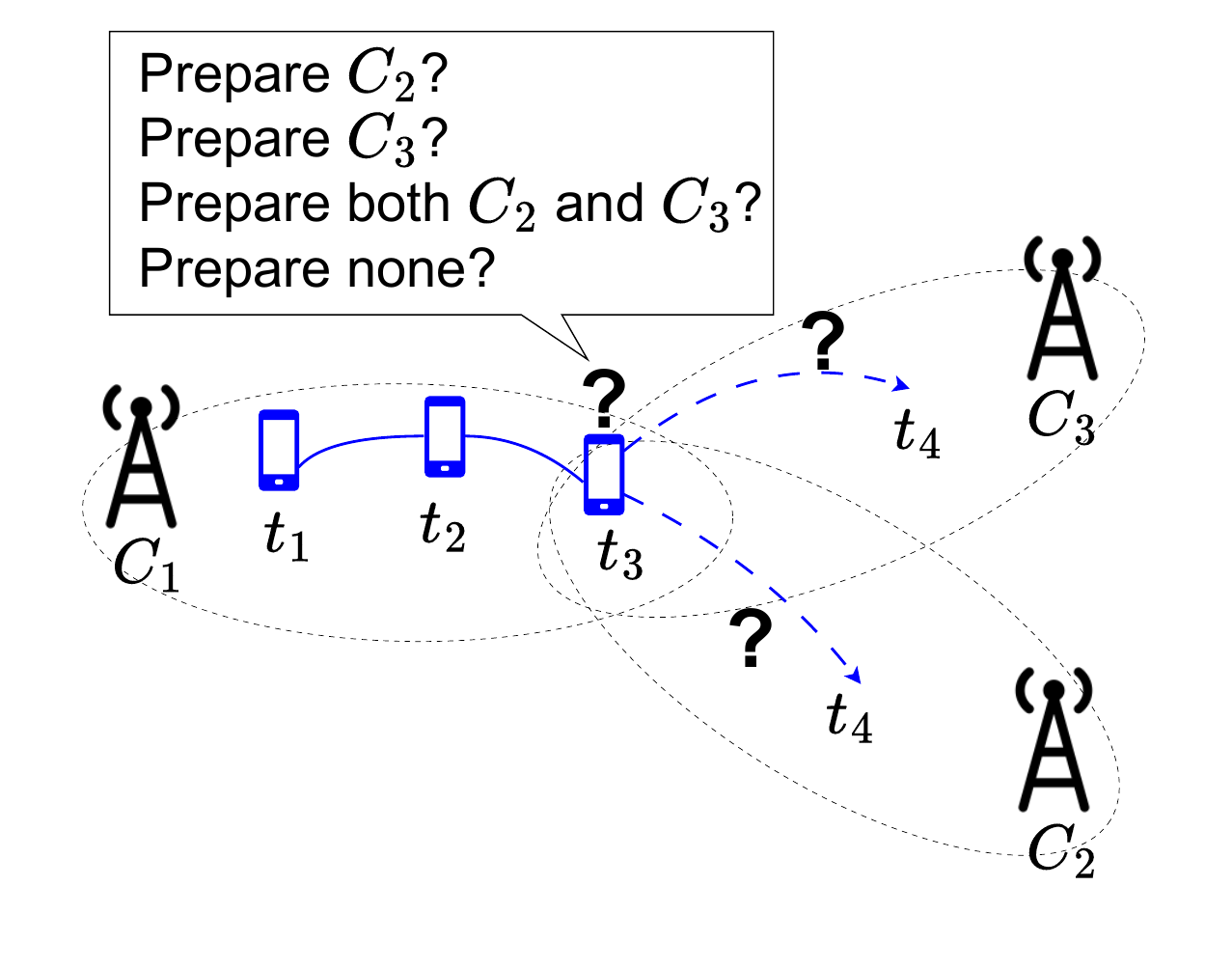}
    \vspace{-7mm}
    \caption{CHO candidate cells for a moving user in consecutive times.}
    \vspace{-8mm}
    \label{fig:CHO_scenarios}
\end{figure}

The goal of this paper is to design a \textit{robust} algorithm that decides which cells to prepare in order to increase the rate of the users and reduce the traditional HO possibility (which comes with increased HOFs and delays), while at the same time minimizing the signaling needed for cell preparations and waste of resources. Robustness in this context refers to maintaining reliable performance without access to accurate information about the signal quality of the users at all times or their mobility patterns.

\subsection{Methods \& Contributions}

After demonstrating the key limitations of traditional HOs and the promising capabilities of CHOs, we introduce a realistic system model that works in time slots, in accordance with online convex optimization (OCO) \cite{hazan-book}, and captures the impact of preparation decisions on both users (UEs) and cells. We consider a central network controller that focuses on providing UEs with high rates and cells with saving their scarce resources, as well as reducing signaling overheads. Similar to other works \cite{prado_CHOopt23}, we provide tunable parameters that prioritize these different CHO criteria based on the preferences of a Mobile Network Operator (MNO). However, for the first time in CHO, ($i$) these parameters can change over time if preferences change, and can be different for each user and cell (e.g., different slices \cite{slicing17}), and ($ii$) assumptions for access to future signal measurements are dropped \cite{kalntis_tcom24}, rendering our modeling applicable to real-world scenarios.

Building on this modeling and problem formulation, we propose a learning algorithm, \texttt{CHOMET}, that updates its preparation decisions dynamically. To ensure \texttt{CHOMET} is adaptive to fast-changing and unpredictable signal conditions, multiple similar (but with different parameters) learners are employed, as well as a meta-learner that tracks and combines their performances, ensuring the best-performing ones are selected. To account for the signaling cost associated with frequent cell preparations and releases (i.e., decision changes across consecutive time slots), we leverage Smoothed Online Learning \cite{zhang-smoothed-ol}. Following a rigorous theoretical analysis, we compare \texttt{CHOMET} through dynamic regret with an omniscient benchmark that has full information on future signal qualities, ensuring that their gap diminishes with time.
Given that the operations of our algorithm are relatively lightweight, it can be implemented in the \textit{near-real-time} of O-RAN \cite{o-ran-andres}, similarly to, e.g., \cite{kalntis22, kalntis_tcom24}. We conclude our study with rigorous experimental evaluations in both ``easy'' (stationary) and volatile scenarios of signal conditions, showing that \texttt{CHOMET} outperforms existing 3GPP-compliant mechanisms.
The key contributions can be summarized as follows:


\noindent$\bullet$ We model, for the first time, CHO as a \textit{smoothed online learning} problem, where the signaling/switching cost can even change over time.

\noindent$\bullet$ We propose a meta-learning algorithm that adjusts cell preparations dynamically, maximizing UEs' rate and minimizing resource waste and signaling overheads. Our algorithm provides strong theoretical dynamic regret guarantees even under volatile signal conditions.

\noindent$\bullet$ We evaluate our algorithm in volatile and slow-changing signal conditions against 3GPP-compliant benchmarks, finding more than 180\% performance improvement.



\section{Background}\label{sec:background}


\subsection{Traditional \& Conditional HOs}

Traditional HOs occur when the signal from the serving cell falls below a threshold (A2 event) or when the signal from a neighboring cell becomes offset better than the signal from the serving cell (A3 event) \cite{3gpp_36_331, 3gpp_38_331, 3gpp_23_401}.
This \textit{reactive} approach, where HOs are initiated \textit{after} the signal conditions are getting/remaining poor for a predefined time, leads to high HO delay and HOF rates, particularly due to the high small cell density and the rapid signal degradation of FR2 (i.e., frequency ranges $>$ 24 GHz) \cite{3gpp_consecutive_CHO, kalntis_imc24}. Specifically, and as can be seen in Fig. \ref{fig:hof_theory}, it is common for HOFs to occur when a user attempts to send a measurement report (MR) under deteriorating signal conditions; or even if the MR is successfully sent, the worsening signal conditions may prevent the user from receiving the subsequent HO command. 

\begin{figure}[!t]
    \centering
    {\label{fig:CHO_HOF_causes}\includegraphics[scale=0.3]{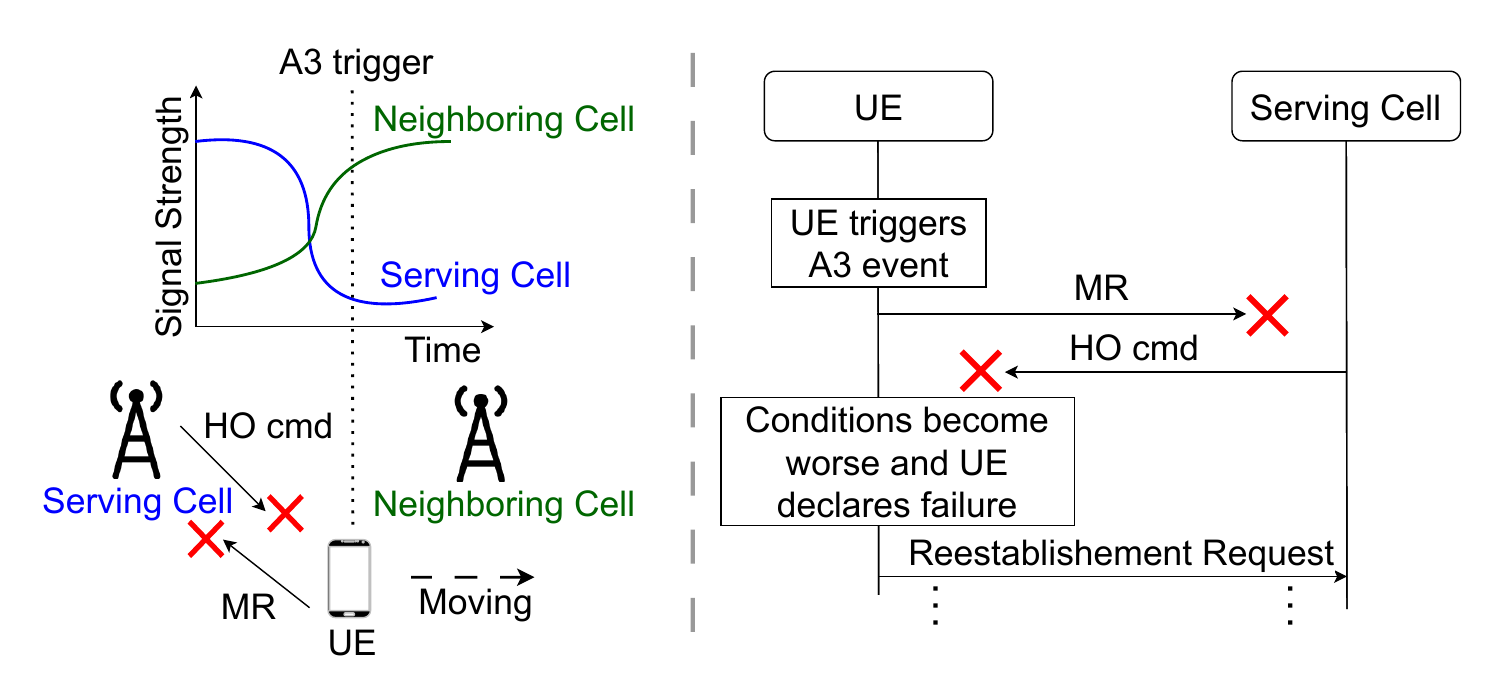}}	
    \vspace{-4mm}
    \caption{A failed traditional HO (left) and its event timeline (right), where the MR or HO command may not reach the serving cell or UE, respectively.}
    \vspace{-8mm}
    \label{fig:hof_theory}
\end{figure}

CHOs, designed as part of 3GPP Release 16 and enhanced in subsequent releases, address the limitations mentioned above by offloading part of the HO decision-making to the user, before the signal conditions deteriorate \cite{3gpp_38_300}. As can be seen in Fig. \ref{fig:CHO_basic_steps}, a CHO decision is taken while the signal quality is still adequate (step 2), and the source cell can pre-configure multiple \textit{candidate} target cells (steps 3, 4, 5) based on the MR sent from the UE (step 1). To conclude the \textit{preparation phase} for the CHOs, the source cell provides monitoring conditions (step 6), such as hysteresis and offset parameters, and the \textit{execution phase} starts: the user applies these conditions continuously (step 7) to evaluate the signal strength and quality --e.g., through signal-to-interference-plus-noise ratio (SINR)-- of the source and candidate target cells. 

If any of the predefined conditions are met (step 8), the UE directly executes the stored HO command (as if this command was just received) without needing to send a MR and await a reply from the source cell: procedures that could delay or/and fail due to the degradation of the signal. If more than a single cell meets the execution condition, the UE decides which cell to access. In our modeling (see Sec. \ref{sec:system_model}), we assume that this cell is the one with the highest SINR adjusted by a cell-specific threshold; however, it is straightforward to consider otherwise. In addition, the UE is not allowed to try the HO execution multiple times or to execute several HO attempts in parallel, towards different candidate cells \cite{stanczak22}. CHO is finalized in the \textit{completion phase}, where resources are released from the source cell, and a new path is established to the target cell (steps 9, 10, and after). We highlight that once the resources are reserved for a UE through the admission control (step 4), they are kept until a cancellation is sent (step 10b).

\begin{figure}[t]
    \centering
    \includegraphics[scale=0.2]{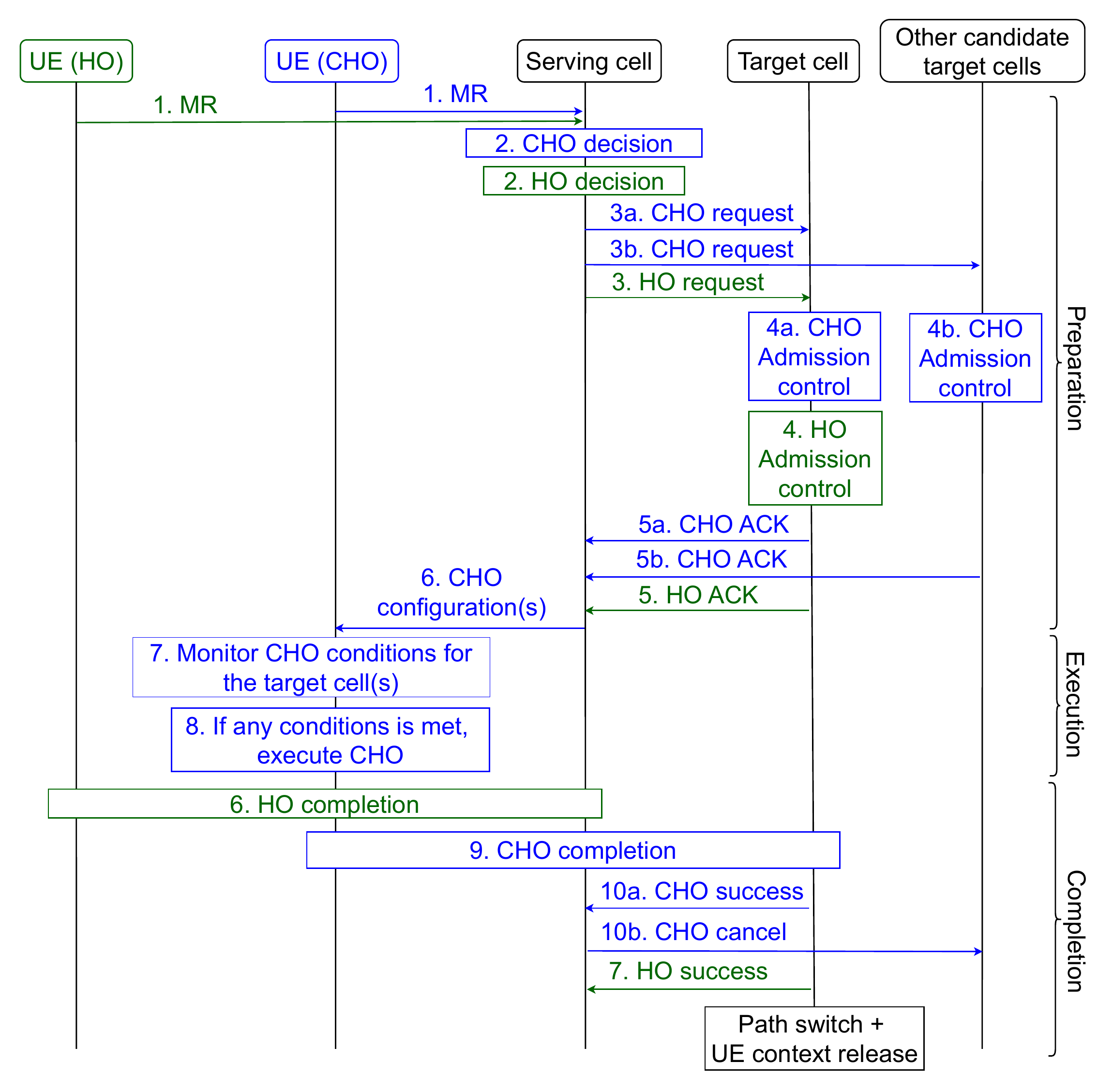}
    \vspace{-3mm}
    \caption{Basic steps/procedures of UEs performing traditional (green) and conditional (blue) handover.}
    \vspace{-6mm}\label{fig:CHO_basic_steps}
\end{figure}

Conversely, in traditional HO, the HO decision (step 2) is made when the signal conditions, included in the MR (step 1), are deteriorating. Moreover, the HO decision in traditional HOs dictates the \textit{one} target cell that the UE should connect to; and this transition (steps 6, 7) is executed \textit{immediately} after the UE receives the command.
In other words, the execution phase in traditional HOs occurs immediately after preparation, in contrast to CHOs where the gap between the two phases can be as large as 9--10 sec. \cite{iqbal22}. 

\subsection{CHO Key Trade-Offs \& Selecting Criteria}\label{sec:system_model_tradeoffs}

Determining which cells to prepare is critical in CHOs for balancing resource efficiency and mobility robustness. Ideally, the MNO, whose goal is to save the scarce and valuable resources of its cells, would allow the preparation of the single cell to which the UE will connect in the next time slot and whose signal strength is high, thus, maximizing the UEs' rate. Since this cannot always be predicted with certainty because, e.g., the user is located at the edge of a cell and/or its trajectory is not known, multiple candidate cells are often prepared. These candidate cells are typically selected based on signal quality metrics (e.g., SINR), which have to exceed a predefined threshold for a specific time \cite{stanczak22} (these thresholds are lower than those for traditional HO). This approach, while straightforward and easy to implement, can lead to inefficient resource usage, as resources of prepared cells cannot be used by another UE until they are released.

At the same time, a long list of prepared cells does not \textit{necessarily} increase the likelihood of including the ``correct'' target cell (i.e., the cell that the user will connect in the next time slot), especially if all prepared cells exhibit low signal quality. Conversely, while preparing fewer cells may conserve resources, it increases the risk of traditional HOs (and thus, more HOFs) to occur, if the correct target cell is not included in the prepared list; preparing no cells leads to traditional HOs. 

Another aspect to consider for optimizing CHOs is the signaling cost of preparing cells. For instance, in environments such as FR2, the small cell density and rapid signal fluctuations result already in more frequent HOs and HOFs, which contribute to significant signaling overhead. The additional burden of continuously preparing and releasing cells due to constant signal variations further exacerbates this overhead \cite{iqbal22, iqbal23_FCHO_hand_blockage}. For that reason, while initial 3GPP releases state that UEs should release CHO candidate cells after (any) successful HO completion to save resources \cite{3gpp_38_300} (see also Fig. \ref{fig:CHO_basic_steps}), follow-up works suggest that doing so is not always optimal in terms of reducing HOFs and signaling overhead \cite{3gpp_consecutive_CHO}.

These trade-offs emphasize that the selection process for cell preparations should focus on finding the balance (and sometimes prioritizing) among identifying the ``best'' --in terms of signal quality-- cells, while keeping the number of prepared cells as small as possible to save resources and avoid frequent preparations and releases to reduce signaling overheads.

\section{System Model \& Problem Formulation}\label{sec:system_model}

\subsection{Modeling Components}

We consider a heterogeneous cellular network comprising a set $\mathcal J$ of $J$ cells serving a set $\mathcal I$ of $I$ UEs. We assume that a central network controller takes decisions for multiple UEs/cells in a time-slotted manner for a set $\mathcal T$ of $T$ slots \cite{andrews-association, kalntis_tcom24} with 10msec--1sec duration for each slot, as realized with xApps in \textit{near-real-time} (10msec--1sec granularity) of O-RAN\cite{o-ran-andres}.
The key metric is the SINR for the signal delivered by cell $j \in \mathcal{J}$ to user $i \in \mathcal{I}$ in slot $t \in \mathcal{T}$:
\begin{align}
	s_{ij}(t)=\frac{q_j\phi_{ij}(t)}{	W_j\sigma^2 + \sum_{k\in \mathcal{B}_j} q_k\phi_{ik}(t)},
\end{align}
where $q_j$ is the transmit power of cell $j$, $\mathcal B_j$ the set of cells that operate in the same frequency as $j$, $\phi_{ij}(t)$ the channel gain (including pathloss, shadowing, and antenna gains), $W_j$ is the bandwidth of cell $j$, and $\sigma^2$ the power spectral density. In line with previous works \cite{andrews-association,kelleler-jsac23, andrews-globecom21, choi-TWC15}, $s_{ij}(t)$ is the average SINR in the slot $t$, since UEs report multiple SINRs during each slot. The maximum possible rate for UE $i$, if it uses cell $j$ exclusively, can be expressed as $c_{ij}(t)=W_j\log_2\big(1+ s_{ij}(t)\big)$. Moreover, we introduce the random variable $u_{j}(t) \in [0,1]$ to show the part of the bandwidth $W_j$ that is available.

We denote by $x_{ij}(t)\!\in\!\{0,1\}$ the preparation decision: $x_{ij}(t)\!=\!1$ means that cell $j$ is prepared in slot $t$ for user $i$, and $x_{ij}(t)\!=\!0$ otherwise. Also, we define the vector $\bm x_t\!=\!\big(x_{ij}(t) \!\in\! \{0,1\}, i\in\mathcal I, j\in\mathcal J\big)$, with the decision set: 
\begin{align*}
	&\mathcal X=\bigg\{ \bm x \in \{0,\!1\}^{I\cdot J}  \  \Big |   \sum_{j\in \mathcal J}{x}_{ij} \leq J,  i\in \mathcal I \bigg\},
\end{align*}
and its convex hull $\mathcal X^{c}\!=\!\text{co}(\mathcal X)$ that relaxes the integrality, i.e., $\bm x \! \in \! [0,1]^{I\cdot J}$. 
For $p_{ij}(t)$ being the probability that cell $j$ will be the highest-rate cell for UE $i$ during slot $t$, then this CHO will be realized with probability $p_{ij}(t) x_{ij}(t)$, where $\bm p_t\!=\!\big(p_{ij}(t) \!\in\! \{0,1\}, i\in\mathcal I, j\in\mathcal J\big)$ and $p_{ij}(t) = \mathds{1}\!\{ j = \arg\max_{k} \big(s_{ik}(t) - o_k(t)\big)\}$, with $o_j(t)$ being the cell-specific offsets that can vary in each slot $t$, resembling the offsets used in, e.g., the A3 event \cite{3gpp_36_331, 3gpp_38_331, martikainen18} to avoid often HOs/CHOs. It is important to note that a user can be served from only one cell at a time, which is inferred from the definition of $\bm p$.

\subsection{Problem Statement}

As mentioned earlier, CHOs are \textit{network-configured, user-decided HOs}. Thus, it is the network controller that decides the preparations of the cells based on its belief regarding the cell that will be the best for each UE, while minimizing resource waste and signaling overheads; and each UE is allowed to execute the CHO directly to one of these prepared cells. Otherwise, a traditional HO occurs, which incurs additional costs due to higher HO delays and HOF probabilities.

With these in mind, we introduce the \textit{utility} function that the network controller wishes to maximize as follows:

\begin{align*}
    g_t(\bm x) \triangleq \sum_{i=1}^I \sum_{j=1}^J & \Big(x_{ij}(t) p_{ij}(t) u_{j}(t) \log c_{ij}(t) \\ &
- \beta_t \ x_{ij}(t) \big(1-p_{ij}(t)\big) \\ &
- \gamma_t \ p_{ij}(t) \big(1-x_{ij}(t)\big)\Big), \bm x \in \mathcal X.
\end{align*}

\noindent The first term of $g_t$ defines the rate of a user, which is non-zero only for its served cell and is discounted by $u_j$ due to exogenous (i.e., independent of the preparation decisions) effects; e.g., other UEs that executed traditional HO. The logarithmic transformation balances the sum-rate across all users to achieve fairness \cite{andrews-association}; however, we note that other mappings (e.g., linear) can be used to capture the specifics of different applications. The second term of $g_t$ introduces a penalty if cells other than the highest-rate ones are prepared to reduce resource waste, and the third term represents an additional cost if the highest-rate cell is not prepared. The scalarization parameters $\beta_t \text{ and } \gamma_t$ are used to normalize units and prioritize one criterion over the other according to the preferences of each MNO that can even change over time or be different for each UE-cell pair due to different slices \cite{slicing17}.

At the same time, the goal of the network controller is to minimize the signaling/switching overheads (or, similarly, maximize their negation), which are captured using $-\delta_t \ \| \bm x_t-\bm  x_{t-1}\|_{B_t}$, with $\bm x_t, \bm x_{t-1} \!\!\in\!\! \mathcal X$. This presents the \textit{switching cost} induced by the signaling of preparing and releasing cells (see Fig. \ref{fig:CHO_basic_steps}) scaled by $b_{ij}(t)$, as each cell $j$ might have different costs for preparing and releasing cells for each user in slot $t$ due to fluctuating traffic demands. More precisely, we define with $\bm B_t\!=\!\text{diag}(\bm b_n(t)\!>\!0)$ a positive definite matrix which has on its diagonal the signaling weights $b_n(t) \in [0,1]$, $n\!=\!i\cdot j$ when UE $i$ prepares and releases cell $j$, and $\|\cdot\|_{B_t}$ is its induced norm that can change over time $t$, i.e., $\|\bm{x}\|_{B_t}^2\!=\! \sum_{n}\! b_n(t) x_n^2$ and its dual $\|\bm{x}\|_{B_t*}^2\!=\!\sum_{n}\!  x_n^2/b_n(t)$ \cite{beck-book}. The role of the parameter $\delta_t$ is similar to $\beta_t \text{ and } \gamma_t$. Thus, the overall problem that the network controller wishes to solve can be expressed as:

\vspace{-4mm}

\begin{align*}
\mathbb{P}: & \max_{\{\bm{x}_t\}_{t}} \sum_{t=1}^T \Big(g_t(\bm x_t) -\delta_t \ \| \bm x_t-\bm  x_{t-1}\|_{B_t} \Big) \\ & 
\textrm{s.t.} \quad \bm{x}_t\in\{0,1\}^{I\cdot J}, \ \forall t \in \mathcal T,
\end{align*}


\noindent where $g_t(\bm x_t) -\delta_t \ \| \bm x_t-\bm  x_{t-1}\|_{B_t}$ is the \textit{objective} function. Solving the problem at the beginning of the horizon $T$ is challenging due to several factors. First, the controller at $t=1$ has no knowledge of the future SINR values for each UE (it has no access or effect on how the UEs will move in the future). In fact, SINRs remain unknown even at the beginning of each slot, as knowing the current SINR and deciding afterwards the cell preparations lead to trivial solutions\footnote{A similar way to approach the problem would be to decide the cell preparations in the previous slot, and then access the SINRs at the beginning of each slot, see \cite{prado_CHOopt23}.}. In other words, the cell preparations should be decided based on the history (one or multiple slots) of the UE's signal, a problem that is further perplexed by the discreteness of the preparation variables \cite{lesage-tac-21}. What is more, the signaling overhead depends on the change in preparation state, as a cell that remains prepared does not incur additional signaling. This introduces a memory effect into the system, as past decisions influence current preparation decisions (due to signaling costs). Finally, we highlight that the $B_t$--norm implies that the switching costs change over time; however, this cost in time $t$ does not change based on previous preparation decisions $\{\bm x_\tau\}, \tau = 1,..., t-1$. 

Therefore, our goal is to design an algorithm that is oblivious to the time-varying and unknown parameters and maximizes the users' rate, while keeping the signaling costs and the amount of wasted resources to a minimum.
\section{Algorithm Design}\label{sec:algorithm_CHOMET}

We approach $\mathbb{P}$ as a \emph{smoothed online learning} problem and address it through \emph{meta-learning} based on the \emph{experts} framework \cite{warmuth-experts, hazan-meta}. More precisely, and due to the problem's unknown variability where optimal preparation decisions can change drastically between slots (due to the variability of the signal conditions), we deploy $K$ similar learning agents, known as \emph{experts}, but each with a different learning step $\theta=(\theta_k, k\in\mathcal K)$ (to ensure at least one will perform well), and a meta-learner whose goal is to find the best-performing expert(s) on-the-fly, while signal conditions might vary. 

More precisely, at the beginning of each slot $t, \text{with } t=1, ..., T$, each expert $k$ shares its suggestion $\bm{x}_t^k \in \mathcal X$ regarding which cells should be prepared for each UE; we assume $x_0^k =0, k \in \mathcal K$. The meta-learner, which has assigned weights (i.e., ``trust'') $\bm{w}_t\!=\!(w_t^k, k\in \mathcal K)$, with $\bm w_t^\top \bm{1}_K\!=\!1$ to each expert, synthesizes them as follows:

\vspace{-4mm}

\begin{align}
\bm x_t^m=\sum_{k\in\mathcal K}w_t^k\bm{x}_t^k.\label{eq:meta-mix}
\end{align}

\noindent To ensure that the decision $\bm x_t^m$ of the meta-learner is implementable (i.e., binary decision variable), we use a quantization function $Q_{{\mathcal X}}$ implementing any unbiased sampling technique, outputting $\bm x_t \in \mathcal X$; here, we opt for Madow sampling \cite{madow}. We underline that, for $\bm{x} \in \mathcal X^c$, $g_t$ is concave (as a linear function), and subtracting the $B_t$--norm preserves the property.

Once cell preparations are finalized and implemented, the controller observes the actual signal qualities from this slot $t$, $s_{ij}(t), \forall i \in \mathcal I \text{ and } j \in \mathcal J$ (we fill with zero the cells that are unreachable for each UE), and calculates the gradient of only the utility function for the ``synthesized'' decision (and not the decision of each expert $\bm x_t^k$), $\nabla g_t(\bm x_t)$, which is sent to all experts. Then, the meta-learner uses the surrogate (i.e., partially linearized) loss:

\vspace{-4mm}

\begin{align}\label{loss-function}
\ell_t(\bm x_t^k)=\langle{\nabla g_t(\bm x_t) },{\bm x_t^k - \bm x_t}\rangle -\delta_t \ \| \bm x_t-\bm  x_{t-1}\|_{B_t}
\end{align}

\noindent to update the weights:

\begin{align}\label{experts-weights-update}
    w_{t+1}^k=\frac{w_{t}^ke^{\eta\ell_t(\bm x_t^k)}}{\sum_{k\in\mathcal K} w_{t}^ke^{\eta \ell_t(\bm x_t^k)} }.
\end{align}


\noindent Lastly, experts update their choices via online gradient ascent:
\begin{align}\label{expert-update}
	\bm{x}_{t+1}^k=\Pi_{\mathcal X}\Big( \bm{x}_t^k + \theta_k\nabla g_t(\bm x_t)\Big),
\end{align}

\noindent where $\eta$ is the meta-learning step. A summarization of the steps can be seen in Algorithm \ref{alg:chomet} (\texttt{CHOMET}).


\begin{algorithm}[t] 
	{\small{	
			\nl \textbf{Required}: Step $\eta$ for meta-learner and $\{\theta_k\}_{k\in \mathcal K}$ for experts \\ 
			\nl \textbf{Initialize}: Sort $\theta_1\leq \theta_2\leq \ldots \leq \theta_K$ and set ${w}_1^{k}=\frac{1+(1/K)}{k(k+1)}, \forall k\in \mathcal K$ \\
			\nl \For{ $t=1,2,\ldots, T$  }{
				\nl Each expert $k\in \mathcal K$ shares its decision $\bm x_t^k$\\
				\nl  The controller synthesizes  $\bm x_t^m$ using \eqref{eq:meta-mix}\\
				\nl  The controller implements $\bm{x}_t=Q_{{\mathcal X}}(\bm x_t^m)$\\		
				\nl The controller sends $\nabla g_t(\bm x_t)$ to experts \\
				\nl The controller updates its weights using \eqref{experts-weights-update} \\							
				\nl Each expert updates its decision using \eqref{expert-update} \\							
			}
			\caption{{\small{\underline{\smash{CHO}} via \underline{\smash{Met}}a-Learning (\texttt{CHOMET})}}}\label{alg:chomet}
	}}	
	\setlength{\intextsep}{0pt} 
\end{algorithm} 

The performance of the algorithm will be assessed using the \emph{Expected Dynamic Regret}, defined as in \cite{zhang-smoothed-ol}:
\begin{align}
\!\!\!\!\mathbb{E}\left[{\mathcal R}_T\right]\!\triangleq \!\sum_{t=1}^T g_t(\bm x_t^{\star}) \!-\! \sum_{t=1}^T \mathbb{E}\Big[ g_t(\bm{x}_t) \!-\!\delta_t \| \bm x_t\!-\!\bm  x_{t-1}\|_{B_t}\Big], \label{regret-metric}
\end{align}
where $\{\bm{x}_t\}_t$ is the algorithm decisions, $\{\bm x_t^{\star}\}_t$ is the benchmark that is the solution of $\mathbb{P}$, and the expectation captures the randomization due to the quantization $Q_{\mathcal X}$. Our goal is to design an algorithm that ensures this gap diminishes with time, $\lim_{T\rightarrow \infty} \mathbb{E}[{\mathcal R}_T]/T\!=\!0$ for any possible benchmark sequence $\{\bm x_t^{\star}\}_t$. In other words, we compare our algorithm with the best oracle that has full information on SINR for all users and cells at the beginning and for the entire horizon $T$. This serves as a highly competitive benchmark, exceeding the static (single best solution for all time slots), or dynamic one (best solution for each time slot independently).

To proceed, we observe that from the problem formulation (see Section \ref{sec:system_model}), the gradients of all functions are bounded by $G$, i.e., $\|\nabla g_t(\bm x)\|_2 \!\leq\! G$, and the diameter of the domain is also bounded by D, i.e., $\|\bm x\!-\!\bm x'\|_2 \!\leq\! D$, where $D = \sqrt{I J}$ and $G=\sqrt{I(J-1) \beta_{\max}^2 + I(\log c_{\max} + \gamma_{\max})^2}$, respectively, with $\beta_{\max} = \max_{t \in \mathcal T} \{\beta_t\}$, $\gamma_{\max} = \max_{t \in \mathcal T} \{\gamma_t\}$, $c_{\max} = \max_{t \in \mathcal T} \{c_t\}$. Therefore, defining $b_{\max}=\max_{n\leq I\cdot J}\{b_n(t)\}$ for $t \in \mathcal T$, it holds that $\|\bm x\!-\!\bm x'\|_{B_t} \!\leq\! D\sqrt{b_{\max}} \!\triangleq \! D_B$, $\|\bm x\!-\!\bm x'\|_{B_t*} \!\leq\! D/\sqrt{b_{\max}}\!\triangleq \! D_{B*}$, and $\|\nabla g_t(\bm x)\|_{B_t} \!\leq G \sqrt{b_{\max}} \!\triangleq\!G_B$.

\begin{lemma}[Performance Analysis / Optimality Guarantee]\label{lemma} Using the parameters:
\begin{itemize}
	\item	$K=\left\lceil\log_2\sqrt{1+\!2T} \right \rceil+1$,

	\item	$\theta_k=2^{k-1}\sqrt{\frac{D_B^2}{T(G^2+2G_B)}}$,  $k=1,\ldots,K$,
				
	\item	$\eta=1/\sqrt{T\nu}$, with $\nu\triangleq (D_B\!+\!1/8)(GD\!+\!2D_B)^2$,	

    \item $P_T\!=\!\sum_{t=1}^T \!\!\|\bm{x}_t^{\star}\!-\bm{x}_{t-1}^{\star}\|_{B_t}$ (path length),
\end{itemize}	
the discrete decisions $\{\bm x_t\}_t$, where $\bm x_t\in \mathcal{X}$ ensure:
\begin{align}
\mathbb{E}\left[{\mathcal R}_T \right]\! \leq \! \notag \sqrt{T}\Big( & \sqrt{\nu}\left(1+\ln(1/w_1^k)\right) + \\ & \sqrt{(G^2+2G_B)(D_B^2+2D_{B*}P_T)} \ \Big) \notag + \\ & \hspace{-6mm} T(G+\sqrt{b_{\max}})\sqrt{IJ}/2 \label{eq:lemma_res},
\end{align}
\end{lemma}

\begin{proof}
The proof follows by tailoring the main result of \cite{kalntis_infocom25, zhang-smoothed-ol, zhang-18}, which, however, solve the UE-cell association problem. Specifically, after setting w.l.o.g. $\beta_t=\gamma_t=\delta_t=1$, eq. \eqref{regret-metric} can be rewritten as:
\begin{align}
    &\mathbb{E}\left[{\mathcal R}_T\right]= \mathbb{E}\Bigg[\!\sum_{t=1}^T\! \Big(g_t(\bm x_t^{\star}) \!- \big(g_t(\bm x_t) \!-\! \|\bm x_t\!-\!\bm  x_{t-1}\|_{B_t}\big) \notag + \\ & \big(g_t(\bm{x}_t^m)\!-\! \|\bm x_t^m\!-\!\bm  x_{t-1}^m\|_{B_t}  \big) - \big(g_t(\bm{x}_t^m)\!-\! \|\bm x_t^m\!-\!\bm  x_{t-1}^m\|_{B_t} \big)\!\Big) \Bigg] \notag \! = \\ & 
    \sum_{t=1}^T \!\! \Big(\!g_t(\bm x_t^{\star}) \! -\! \big(g_t(\bm{x}_t^m)\!-\! \|\bm x_t^m\!-\!\bm  x_{t-1}^m\|_{B_t} \big)\!\Big) \!\!+\! 
    \mathbb{E}\!\Bigg[\!\sum_{t=1}^T\! \Big(\!\big(g_t(\bm{x}_t^m)-\! \notag \\ & \|\bm x_t^m\!-\!\bm  x_{t-1}^m\|_{B_t}\big) \!-\! \big(g_t(\bm x_t) \!-\! \|\bm x_t\!-\!\bm  x_{t-1}\|_{B_t}\big)\!\Big)\!\Bigg]. \label{eq:prep_and_discr_err}
\end{align}

\begin{figure*}[!t]
    \centering
    \subfigure[]{\label{fig:num_changes_best_x_500-beta_0.1-T_5000}\includegraphics[scale=0.65]{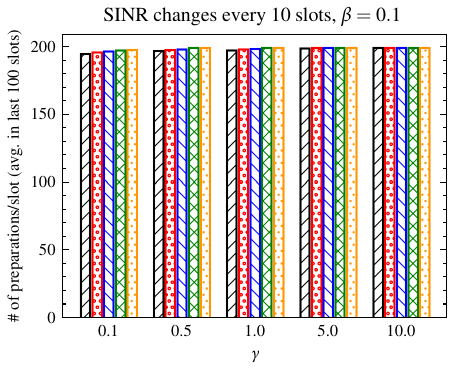}}	
    \hspace{0.5cm}
    \subfigure[]{\label{fig:num_changes_best_x_500-beta_0.3-T_5000}\includegraphics[scale=0.65]{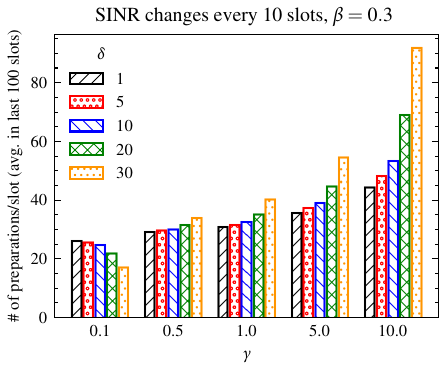}}
    \hspace{0.5cm}
    \subfigure[]{\label{fig:num_changes_best_x_500-beta_0.5-T_5000}\includegraphics[scale=0.65]{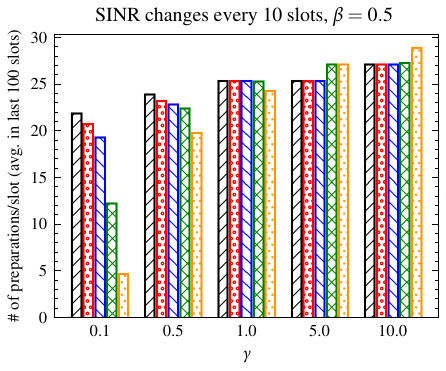}}
    \vspace{-2mm}
    \caption{Effect of parameters $\beta_t\equiv\beta, \gamma_t\equiv\gamma, \text{and } \delta_t\equiv\delta, \text{ for } T=5k$ slots in the volatile scenario (SINR changes ever 10 slots) when running \texttt{CHOMET}: (a) $\beta=0.1$, (b) $\beta=0.3$, and (c) $\beta=0.5$. Y-axis symbolizes the average number (last 100 slots) of preparations per slot.}    
    \vspace{-5mm}\label{fig:num_changes_best_x_500-T_5000}
\end{figure*}

\vspace{-3mm}

Therefore, to characterize the performance of \texttt{CHOMET}, we begin by bounding the expected dynamic regret of the relaxed (continuous) preparation decisions $\{\bm x_t^m\}_t$, $\bm x_t^m \!\!\in\!\! \mathcal{X}^c$, and afterwards, the error of the (implementable) discrete decisions, which correspond to the first and second terms of eq. \eqref{eq:prep_and_discr_err}, respectively. As can be seen from the first term of the result in eq. \eqref{eq:lemma_res}, sublinear dynamic regret can be achieved for the relaxed preparation decisions. It follows by bounding the regret of each expert w.r.t. the benchmark and then the regret of the meta-learner w.r.t. any expert. For the extra cost/error introduced due to the discretization of the preparation decision through the quantization routine $Q_{{\mathcal X}}$, we bound the second term of eq. \eqref{eq:prep_and_discr_err} as follows:
\[\mathbb{E}\!\!\left[ \!\sum_{t=1}^T\!\! \Big(\!g_t(\bm{x}_t^m)\!-\! \|\bm x_t^m\!-\!\bm  x_{t-1}^m\|_{B_t} \!-\!g_t(\bm x_t) \!+\! \|\bm x_t\!-\!\bm  x_{t-1}\|_{B_t}\! \Big)\!\right] \!\!\stackrel{(*)}\leq \]
\[(G+\sqrt{b_{\max}})\sum_{t=1}^T  \sqrt{\mathbb{E}\left[ \sum_{i=1}^I\sum_{j=1}^J\Big({x}_{ij}(t) -\mathbb{E}[{x}_{ij}(t)] \Big)^2\right]} \stackrel{(**)}\leq\]
\[\hspace{-5.2cm}T(G+\sqrt{b_{\max}})\sqrt{IJ}/2,
\]

\noindent leading to eq. \eqref{eq:lemma_res}. Specifically, for $(*)$, we use Jensen's inequality, the linearity of expectation, and the Lipschitz continuity. In particular, for the Lipschitz continuity, considering two points $\bm x_1, \bm x_2 \in \mathcal{X}^c$, using the triangle inequality and the Lipschitz constant $G$ of $g_t$ and of the Euclidean norm, and defining $\bm x_{t-1}=C$, we have:
\begin{align*}
    \Big|&\big(g_t(\bm{x}_1) -  \|\bm{x}_1 - C\|_{B_t} \big) - \big(g_t(\bm{x}_2) - \|\bm{x}_2 - C\|_{B_t} \big)\Big| \leq \\ &
    \big|g_t(\bm{x}_1) - g_t(\bm{x}_2)\big| + \big|\|\bm{x}_1 - C\|_{B_t} - \|\bm{x}_2 - C\|_{B_t}\big|  \leq \\ &
    \big(G+ \sqrt{b_{\max}}\big) \|\bm{x}_1 - \bm{x}_2\|.
\end{align*}

\noindent Lastly, we notice that $(**)$ calculates the variance of the binary $\bm x_t$. Given that $\bm x_t \in \{0,1\}^{I \cdot J}$ and the quantization routine $Q_{\mathcal X}$ implements Madow sampling, the maximum variance each component can obtain is $1/4$ (the variance per element is $q(1-q)$ due to the Bernoulli trial with probability $q$ for each $x_{ij}$, and maximizes for $p=0.5$). Hence, we can upper bound the expression in the square root with the constant $IJ/4$.
\end{proof}

With Lemma \ref{lemma}, we bound the expected regret of the implementable cell preparations, in contrast to previous works \cite{andrews-association, kelleler-jsac23}. Even though the relaxed/continuous preparations achieve sublinear dynamic regret, the discretization introduces an unavoidable non-diminishing error. As has been shown in the UE-cell association context \cite{kalntis_infocom25}, this error is small in practical scenarios, and the algorithm converges towards the optimal solutions (see also Sec. \ref{sec:evaluation}).

\section{Performance Evaluation}\label{sec:evaluation}

We assess \texttt{CHOMET} under different synthetic scenarios to verify its robustness and showcase its learning convergence. For that, we compare our algorithm against the baseline 3GPP-compliant HO/CHO algorithm using the A3 event, which is the algorithm currently being used by MNOs and antenna vendors \cite{3gpp_38_300, 3gpp_38_331}. 
We use the tuple notation \texttt{(\# Best BS, TTT)} to refer to these comparators, where the first argument refers to the number of top-$N$ best cells (i.e., highest SINR) that are prepared for each user in each slot, while the second is the time-to-trigger (TTT, i.e., number of consecutive slots a cell must have been in the top-$N$ before is prepared), resembling the A3 event. For example, algorithm \texttt{(3, 8)} prepares the 3 highest-SINR cells for each user only if these cells remain the highest ones for at least 8 slots. Moreover, we compare \texttt{CHOMET} with an optimal \texttt{Oracle} that solves the optimization problem in every step using \textit{CVXPY} \cite{diamond2016cvxpy}. We note that the comparison with the oracle that has complete knowledge of the future is computationally intensive for mixed-integer programs \cite{kelleler-jsac23, Bragin2022}; however, our dynamic regret guarantees hold, as can be verified from Sec. \ref{sec:algorithm_CHOMET}.

In line with prior works \cite{kalntis_tcom24, kalntis_infocom25}, we select two synthetic scenarios as follows, for $T=5k$ slots: ($i$) \textit{stationary}: the SINR $s_{ij}(t), i \in \mathcal I, j \in \mathcal J$ remains almost constant across all slots $t \in \mathcal T$, changing only once every 600 slots, and ($ii$) \textit{volatile}: $s_{ij}(t)$ fluctuates every 10 slots within the range of [10, 30]dB \cite{3gpp_36_133}, encompassing poor to excellent values. In both scenarios, we randomly select the bandwidths $W_j \in \{5, 10, 15, 20\}$ MHz \cite{3gpp_36_101}, while $\bm B_t$ takes random values within [0, 1], as the offsets $o_j(t), \forall j \in \mathcal J $. We select $I\!=\!20$ UEs and $J\!=\!10$ cells to facilitate the calculation of the average regret, as determining the best oracle is computationally intensive. However, we underline that the best oracle is not needed to run \texttt{CHOMET}.

First, Fig. \ref{fig:num_changes_best_x_500-T_5000} shows the effect of the parameters $\beta_t, \gamma_t, \text{and } \delta_t$ on \texttt{CHOMET}, which, for simplicity, are assumed to be constant for the entire duration of the experiments, namely $\beta_t\equiv\beta, \gamma_t\equiv\gamma, \delta_t\equiv\delta, \text{and } u_t=1, \forall t \in \mathcal T$, where $T=5k$ and SINR changes every 10 slots (volatile scenario). It is important to note that due to the different values that each component of the objective function takes, we focus on the comparison of the values within the same parameter (i.e, they act as scalarization values too). For example, $\beta=0.5 > \gamma = 5$ does not imply that more importance is given to $\gamma$; however, choosing $\beta=0.5$ instead of $\beta=0.1$ does. We choose $\beta = \{0.1, 0.3, 0.5\}$, $\gamma = \{0.1, 0.5, 1, 5, 10\}$ and $\delta = \{1, 5, 10, 20, 30\}$.  The y-axis of Fig. \ref{fig:num_changes_best_x_500-T_5000} shows the average number of preparations per slot, for the last 100 slots; notice that, due to considering 20 UEs and 10 cells, this number cannot exceed 200.

For a specific $\gamma$, the effect of $\beta$ and $\delta$ on the number of prepared cells depends on the interplay of these parameters and should be carefully examined. If $\gamma$, the penalty for not preparing the best cell, is small (e.g., $\gamma = 0.1$) and $\beta$ is high (e.g., $\beta = 0.5$ so fewer cells can be prepared), then as $\delta$ increases, the switching cost becomes more significant, making it less beneficial to prepare different cells to determine the best one. Consequently, the number of prepared cells \textit{decreases}; e.g., setting $\beta = 0.5$ and $\gamma=0.5$, we observe 24 and 19 preparations on average for $\delta=1$ vs $\delta=30$, respectively. Conversely, if $\gamma$ is large, the penalty for not preparing the best cell is substantial, and as $\delta$ grows and switching becomes more costly, the algorithm decides to keep more cells prepared for more slots. As a result, the number of prepared cells \textit{increases} even up to $\times$2 times, as we can see for $\beta = 0.3$, $\gamma=10$, and $\delta=1$ vs $\delta=30$.

For a fixed $\delta$, the number of prepared cells \textit{increases} with $\gamma$, as higher $\gamma$ imposes a greater penalty for not preparing the highest-SINR cell. To mitigate this penalty, the algorithm prepares more cells to ensure that the best is found. This trend is consistent for all three $\beta$ values considered, although a higher $\beta$ limits the number of cells other than the highest-SINR that can be prepared, leading to differences in absolute values. For instance, our algorithm makes approximately 200 preparations on average per slot in Fig. \ref{fig:num_changes_best_x_500-beta_0.1-T_5000}, meaning that all cells for all users are prepared in each slot, while we observe at most 29 preparations per slot (i.e., 1--2 cells per UE) in Fig. \ref{fig:num_changes_best_x_500-beta_0.5-T_5000}; for that reason, lower values than $\beta=0.1$ or higher than $\beta=0.5$ are not considered in the evaluation. We underline that the interplay of these parameters also changes significantly with the volatility/stationarity of the conditions (i.e., SINR).

\begin{figure}[!t]
    \centering
    \subfigure[]{\label{fig:regret_num_changes_best_x_500-beta_0.5-T_5000}\includegraphics[scale=0.52]{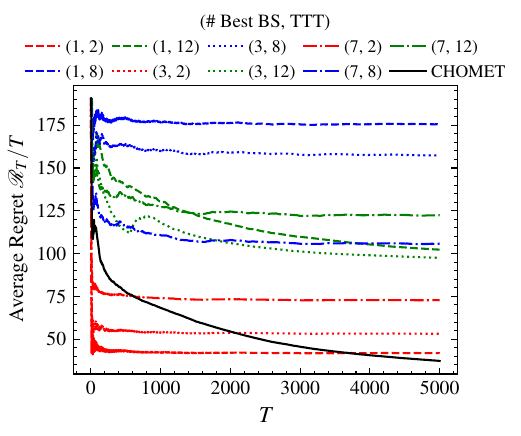}}	
    \subfigure[]{\label{fig:f_total_num_changes_best_x_500-beta_0.5-T_5000}\includegraphics[scale=0.52]{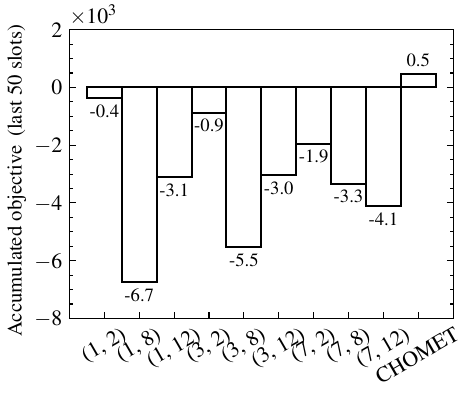}}
    \vspace{-5mm}
    \caption{Volatile scenario (SINR changes ever 10 slots) for $\beta_t = 0.5, \gamma_t = 10, \text{and } \delta_t = 5, \forall t \in \mathcal T$ with $T=5k$ slots: (a) average dynamic regret and (b) total objective values for the last 50 slots, of \texttt{CHOMET} and benchmarks.}
    \vspace{-7mm}
    \label{fig:regret-f_total_num_changes_best_x_500-beta_0.5-T_5000}
\end{figure}

In the volatile case, Fig. \ref{fig:regret_num_changes_best_x_500-beta_0.5-T_5000} shows the average dynamic regret of our proposed algorithm and the 3GPP-compliant competitors, with the former surpassing all \texttt{(q, r)}, for $q=\{1, 3, 7\}$ and $r=\{8, 12\}$ (blue and green lines) by up to 375\% in slot $t=5k$. Even though at first glace competitors \texttt{(1, 2)}, \texttt{(3, 2)} and \texttt{(7, 2)} seem to have comparable performance with \texttt{CHOMET}, we underline that their average dynamic regret stays almost constant for all slots (``stuck'' in sub-optimal decisions). This claim can be verified from Fig. \ref{fig:f_total_num_changes_best_x_500-beta_0.5-T_5000}, where the accumulated objective in the last 50 slots of the best comparator, namely \texttt{(1, 2)}, is 180\% less than \texttt{CHOMET}. Lastly, in a stationary (almost static) case, such as this of Fig. \ref{fig:regret-f_total_num_changes_best_x_5-beta_0.5-T_5000} where SINR changes very slowly (5 times in the total $T=3k$), comparators \texttt{(1, 2)}, \texttt{(1, 8)} and \texttt{(1, 12)} behave similarly to \texttt{CHOMET} in terms of average regret and total objective values. This is reasonable, as preparing only the single highest-SINR cell (which is the one that the user is allocated) is the best policy to follow when conditions stay the same. Therefore, we verify with different scenarios that our algorithm, \texttt{CHOMET}, is adaptive to both volatile and stationary/static cases and provably approaches the behavior of an omniscient benchmark.

\section{Related Work}\label{sec:related_work}

CHO is a novel solution included in 3GPP Release 16 that focuses on mobility robustness by reducing the HO delay and the increased number of HOFs that traditional HOs are facing \cite{3gpp_38_300, survey_CHO_alraih2023, survey_cho_haghrah2023}; kindly refer to \cite{kalntis_imc24, kalntis_infocom25} for a comprehensive literature review of traditional HOs. Recently, CHO has been shown to be successful in several scenarios, e.g., in non-terrestrial networks \cite{leo_juan22, leo_saglam23, leo_yang24}, for integrated access and backhaul, and in 5G NR-unlicensed \cite{stanczak22}, as well as in beamforming and contention-free random access \cite{iqbal_beamforming23, stanczak23_cfra}. In \cite{martikainen18}, by tweaking the CHO thresholds (also studied in \cite{samsung_cho24}) for preparing cells and executing HO towards a prepared cell, the authors discover that CHOs offer mobility robustness (more than 60\% reduced failures); however, improper parametrization could lead to increased signaling. Similarly, \cite{stanczak22} discusses the basics of this new type of HO and reports that UEs with velocities greater than 30km/h benefit the most from CHO, exhibiting 3--4 times lower failures than UEs with no CHO enabled (i.e., traditional HOs). However, these works offer either solely descriptive and simulation-based insights or propose solutions without performance guarantees. 

\begin{figure}[!t]
    \centering
    \subfigure[]{\label{fig:regret_num_changes_best_x_5-beta_0.5-T_3000}\includegraphics[scale=0.52]{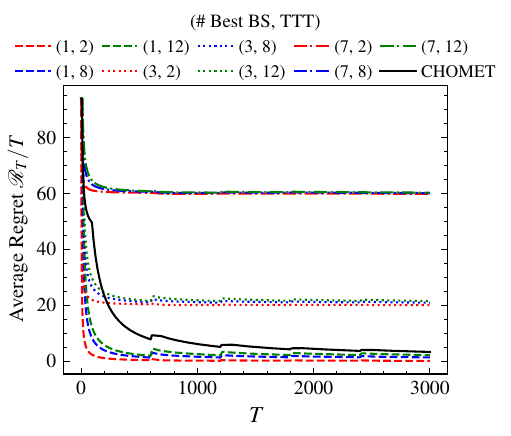}}	
    \subfigure[]{\label{fig:f_total_num_changes_best_x_5-beta_0.5-T_3000}\includegraphics[scale=0.52]{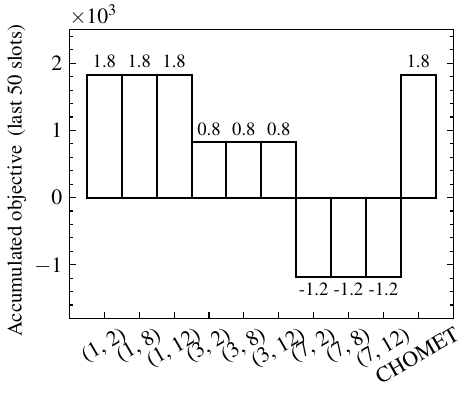}}
    \vspace{-2mm}
    \caption{Stationary scenario (SINR changes every 600 slots) for $\beta_t\!=\!0.5, \gamma_t \!=\! 10, \text{and } \delta_t \!=\! 5, \forall t \!\in\! \mathcal T$ with $T\!=\!3k$ slots: (a) average dynamic regret and (b) total objective values for the last 50 slots, of \texttt{CHOMET} and benchmarks.}
    \vspace{-5mm}
    \label{fig:regret-f_total_num_changes_best_x_5-beta_0.5-T_5000}
\end{figure}

In \cite{fiandrino23, stanczak23_history}, historical HO data (e.g., the likelihood of each cell becoming the next serving cell for a given UE) are used, and the cells to prepare for each user are determined based on past HO frequencies. To enhance CHOs, trajectory prediction is used in \cite{prado_echo21}, where the controller decides which cells should be prepared/released based on the predicted distance between the UEs and the cells. Other prediction-based CHO schemes take as input the signal quality of each user at different slots and rely on Machine/Deep Learning (ML/DL) models to either predict the probabilities of the user's connected (i.e., associated) cell in the next time slot \cite{lee_DL20}, or the future signal measurements \cite{park24}. However, these approaches rely on static historical data, which do not account for changes in user mobility patterns over hours/days \cite{kalntis_imc24}, or require frequent reconfiguration to adapt to such changes, limiting their flexibility in dynamic environments. 

Moreover, the authors of \cite{iqbal22, iqbal23_FCHO_hand_blockage} study the benefits of fast (i.e., consecutive) CHO proposed in \cite{3gpp_consecutive_CHO} to facilitate the use of previously prepared cells and avoid the unnecessary signaling to prepare the same cells, after (any) successful HO. In our work, we introduce tweakable parameters based on the preferences of the MNOs that determine the signaling cost weights, allowing flexibility in cell preparation decisions. Higher signaling weights penalize heavily the frequent preparations, capturing the idea of \cite{3gpp_consecutive_CHO, iqbal22, iqbal23_FCHO_hand_blockage}, while lower weights allow them to occur often (i.e., even after any successful HO) \cite{3gpp_38_300}. Similarly to our work, the authors of \cite{prado_CHOopt23} form an optimization problem to decide which cells to prepare for each user in order to minimize CHO costs. Our study differs from this work, as $(i)$ we learn on the fly the parameters and do not require offline training, and $(ii)$ our perturbation model is an adversarial one, where the various random parameters can even be selected by an attacker; still, all results hold.



\section{Conclusions}\label{sec:conclusions}

This study focuses on optimizing Conditional Handovers (CHOs), a novel type of handovers (HOs) introduced by 3GPP to mitigate the issues of traditional HOs, using Smooth Online Learning (SOL). We develop a dynamic meta-learning algorithm that maximizes the users' rate, while at the same time minimizing the signaling needed for cell preparations and the waste of resources, using tunable parameters for flexible and customizable prioritization of these criteria. Our proposed algorithm, \texttt{CHOMET}, aligned with the O-RAN paradigm, demonstrates robust theoretical performance guarantees even in challenging environments, as exemplified in its evaluation against the currently deployed 3GPP-compliant algorithms. Future works will focus on reducing the discretization error and validating the algorithm in real-world scenarios.
\section*{Acknowledgments}
We sincerely thank Naram Mhaisen and Dr. Andra Lutu for the insightful discussions. This work has been supported by the National Growth Fund through the Dutch 6G flagship project ``Future Network Services'', and the European Commission through Grant No. 101139270 (ORIGAMI) and 101192462 (FLECON-6G).

\appendices
\ifCLASSOPTIONcaptionsoff
  \newpage
\fi
\bibliography{references_V2.bib}
\bibliographystyle{IEEEtran}


\end{document}